\documentclass[11pt]{article}
\usepackage{style}
\usepackage[left=1in,bottom=1in,top=1in,right=1in]{geometry}
\hypersetup{%
	colorlinks   = true, 
	urlcolor     = blue, 
	linkcolor    = blue, 
	citecolor    = blue,  
}

\title{Multiclass Learnability Does Not Imply Sample Compression}
\author{Chirag Pabbaraju\thanks{Stanford University. Email: \texttt{cpabbara@cs.stanford.edu.}}}
\date{\today}

\begin{document}

\maketitle

\begin{abstract}
A hypothesis class admits a sample compression scheme, if for every sample labeled by a hypothesis from the class, it is possible to retain only a small subsample, using which 
the labels on the entire sample can be inferred. The size of the compression scheme is an upper bound on the size of the subsample produced. Every learnable binary hypothesis class (which must necessarily have finite VC dimension) admits a sample compression scheme of size only a finite function of its VC dimension, independent of the sample size. For multiclass hypothesis classes, the analog of VC dimension is the DS dimension. We show that the analogous statement pertaining to sample compression is not true for multiclass hypothesis classes: every learnable multiclass hypothesis class, which must necessarily have finite DS dimension, does not admit a sample compression scheme of size only a finite function of its DS dimension.
\end{abstract}

\section{Introduction}
\label{sec:intro}

Sample compression is a widely studied paradigm in learning theory. At a high level, the main question that sample compression asks is: given a labeled training dataset, is it possible to get by working only with a small fraction of the dataset? A valid \textit{sample compression scheme} gets rid of all the uninformative points in the dataset, and \textit{compresses} the sample to a much smaller subsample, such that there exists an algorithm that can \textit{reconstruct} all the labels on the original sample correctly just from the compressed sample. A classical example of sample compression is exhibited by \textit{support vector machines} for the task of classifying linearly separable data. Here, the compressor may only send the support vectors in the data to the reconstructor. The reconstructor goes on to build a hyperplane that maximally separates the support vectors with the largest possible margin; this, in turn, also recovers correct labels on the non-support-vector points. In the language of learning theory, if such compression-reconstruction is possible for every sample \textit{realizable} by a hypothesis class $\barH$, we say that the class $\barH$ admits a sample compression scheme. In this case, the \textit{size} of the compression scheme is the size $k(m)$ that a sample of size $m$ gets compressed down to. 

In fact, sample compression is intrinsically tied up with the \textit{learnability} of binary hypothesis classes (where the label space is $\{0,1\}$). Formally, \cite{littlestone1986relating} showed that every binary class $\barH$ that admits a sample compression scheme of size $k(m)$ also defines a PAC (Probably-Approximately-Correct) \cite{valiant1984theory} learning algorithm for the class having sample complexity $O(k(m))$. Thus, compression implies learnability in the case of binary classes. In their work, \cite{littlestone1986relating} also asked the converse: does learnability imply compression? Since the PAC learnability of a binary class $\barH$ is completely characterized by the finiteness of its VC dimension $\VC(\barH)$ \cite{vapnik1974theory, vapnik2015uniform, blumer1989learnability}, this question is equivalent to asking: does every binary class $\barH$ having finite VC dimension $\VC(\barH)$ admit a sample compression scheme of size only a finite function of $\VC(\barH)$?

A long line of insightful works on this question culminated with \cite{moran2016sample} answering it in the affirmative. For any binary class $\barH$ having VC dimension $d_{\VC}$, \cite{moran2016sample} constructed a sample compression scheme of size $2^{O(d_{\VC})}$. Prior to their work, existing sample compression schemes had a dependence either on the sample size $m$ (e.g., compression of size $O(d_{\VC}\cdot\log(m))$ via boosting due to \cite{freund1995boosting, freund1997decision}), or on the size of $\barH$ (e.g., compression of size $O(2^{d_\VC}\cdot\log\log|\barH|)$ due to \cite{moran2017teaching}).\footnote{For a detailed and exhaustive list of other prior compression schemes, we refer the reader to Section 1.2.2 in \cite{moran2016sample}.} The work of \cite{moran2016sample} gets rid of both these dependencies and obtains a compression scheme of size only a function of the VC dimension, thus establishing the equivalence of learnability and sample compression for binary hypothesis classes. It is worth mentioning that constructing sample compression schemes of size even sub-exponential in $d_{\VC}$ is still open, and has been a longstanding famous problem in learning theory \cite{warmuth2003compressing}!

For essentially the same reasons that compression implies learnability in the binary case, compression also implies learnability in the \textit{multiclass} case (where the label space is not just $\{0,1\}$ but much larger, possibly infinite too), as was observed by \cite{david2016statistical}. Here too, we can ask the counterpart of \cite{littlestone1986relating}'s question: does multiclass learnability imply sample compression? In fact, the notion of what learnability means in the multiclass setting was only fully established in a recent seminal work by \cite{brukhim2022characterization}, who equated PAC learnability of a class with finiteness of its DS dimension, which was first introduced in the work of \cite{daniely2014optimal}. Concretely, while finiteness of the DS dimension was known to be necessary for learnability, \cite{brukhim2022characterization} also constructed an algorithm that successfully learns classes having finite DS dimension. Thus, the natural question to ask is: does every multiclass hypothesis class $\barH$ having finite DS dimension $\DS(\barH)$ admit a sample compression scheme of size only a finite function of $\DS(\barH)$?

Interestingly, the route that \cite{brukhim2022characterization} take to construct a learning algorithm for hypothesis classes having finite DS dimension \textit{is} via sample compression. Concretely, for any sample of size $m$ realizable by a hypothesis class of DS dimension $d_{\DS}$, they construct a sample compression scheme of size $\tilde{O}(d_{\DS}^{1.5}\cdot \polylog(m))$. This $\polylog(m)$ dependence on the size of the compression scheme is indeed reminiscent of the analogous dependence in the boosting-based compression scheme of \cite{freund1995boosting, freund1997decision} for binary classes. Given that this dependence was ultimately removed in the work of \cite{moran2016sample}, could it also be altogether gotten rid of in the multiclass case?


We answer this question in the negative, and show that a dependence on the sample size $m$ is indeed necessary in the compression size for any valid sample compression scheme in the multiclass setting. Our main result is the following:
\begin{theorem}[Multiclass Learnability $\not \Rightarrow$ Compression]
    \label{thm:compression-lb}
    There exists a hypothesis class $\barH$ mapping a domain $\mcX$ to $\mcY=\{0,1,2,\dots\}$ that satisfies:
    \begin{enumerate}
        \item[(1)] $d_{\DS}(\barH)=1$.
        \item[(2)] Any sample compression scheme for $\barH$ that compresses labeled samples of size $m$ to a subsample of size $k(m)$ must satisfy $k(m)=\Omega((\log(m))^{1-o(1)})$, where the $o(1)$ term goes to 0 as $m \to \infty$. 
    \end{enumerate}
\end{theorem}
This result means that unlike the binary case, we cannot hope to obtain a sample compression scheme  in the multiclass setting where the size of the scheme is a finite function of only the DS dimension of the hypothesis class. Instead, the size of any compression scheme must necessarily depend on the sample size. Note again that (1) above implies $\barH$ is learnable. Therefore, while compression implies learnability, learnability does not imply sample compression in the multiclass case, thus exhibiting a separation amongst the two paradigms in the binary and multiclass case. The rest of the paper is devoted to establishing \Cref{thm:compression-lb} and discussing the result.
\section{Preliminaries and Background}
\label{sec:prelim}

The input data domain is denoted as $\mcX$ and the label space as $\mcY$. Concepts and hypotheses are interchangeably used to mean the same object. To prove our result, we will require dealing with \textit{partial} concept classes $(\mcX \to \{\mcY \cup \{\star\}\})$, where $\star$ is a special symbol, and hence we will denote the otherwise standard \textit{total} concept classes $(\mcX \to \mcY)$ with symbols having a bar on top (e.g., $\barH$) and partial classes without a bar (e.g., $\mcH$). When we are thinking of partial classes and want the label space to additionally also include the special symbol $\star$, we will be explicit and use $\{\mcY \cup \{\star\}\}$ --- otherwise, $\mcY$ should be assumed to not include $\star$. For a sequence $S \in \mcX^d$, we denote the restriction of a class $\mcH$ on $S$ (all the different ways in which members of $\mcH$ can label $S$) by $\mcH|_S$.

\subsection{Partial Concept Classes and Disambiguation}
\label{subsec:pcc}
Our lower bound heavily relies on the theory of partial concept classes introduced in the work of \cite{alon2022theory}. Concepts in a partial concept class are allowed to be undefined in certain regions of the input domain. These regions vary based on known structural assumptions on the data like margin-separatedness, data lying on a low-dimensional subspace, etc. 
\begin{definition}[Partial Concept Classes \cite{alon2022theory}]
    \label{def:pcc}
    Given an input space $\mcX$, a label space $\mcY=\{0,1,2,\dots\}$, and a special symbol $\star$, a partial concept class $\mcH$ maps $\mcX$ to $\mcY \cup \{\star\}$ i.e., $\mcH \subseteq \{ \mcY \cup \{\star\}\}^\mcX$. For any $h \in \mcH$, if $h(x)=\star$, we say that $h$ is undefined at $x$. The support of a partial concept $h \in \mcH$ is defined as $\supp(h)=\{x \in \mcX:h(x)\neq\star\}$, and $\supp(\mcH)=\bigcup_{h \in \mcH}\supp(h)$. A labeled sequence $S = \{(x_1,y_1),\dots,(x_m,y_m)\}$ is realizable by $\mcH$ if there exists a partial concept $h \in \mcH$ such that $\forall i \in [m],~ h(x_i) \neq \star \text{ and } h(x_i) = y_i$. 
\end{definition}

If every concept in the partial concept class has full support, the $\star$ symbol becomes irrelevant and we get the usual notion of a total concept class.
\begin{definition}[Total Concept Classes]
    \label{def:tcc}
    A partial concept class $\barH \subseteq \{ \mcY \cup \{\star\}\}^\mcX$ that satisfies $\supp(\barh) = \mcX,~ \forall \barh \in \barH$, is a total concept class.
\end{definition}
Total concept classes naturally define the notion of ``disambiguation'' of partial concept classes.
\begin{definition}[Disambiguation, Definition 9 in \cite{alon2022theory}]
    \label{def:disambiguation}
    A total concept class $\barH \subseteq \mcY^\mcX$ disambiguates a partial concept class $\mcH$ if for every finite labeled sequence $S=\{(x_1, y_1),\dots,(x_m,y_m)\}$ realizable by $\mcH$, there exists $\barh\in \barH$ such that $\forall i \in [m],~ \barh(x_i)=h(x_i)$.
\end{definition}
The hard-to-compress hypothesis class that realizes our lower bound in \Cref{thm:compression-lb} will be a suitable disambiguation of a hard-to-compress partial concept class. Next, we define the relevant complexity parameter that completely captures learnability of a multiclass hypothesis class --- the DS dimension.

\subsection{DS Dimension}
\label{subsec:ds}

As mentioned above, while the VC dimension of a binary (total) concept class was long known to completely characterize its learnability, the corresponding problem of characterizing the learnability of a class on multiple classes was only resolved recently in the work of \cite{brukhim2022characterization}. They showed that a combinatorial complexity parameter called the DS dimension (due to \cite{daniely2014optimal}) is the appropriate equivalent of the VC dimension in terms of characterizing learnability of multiclass concept classes. Since total classes are special cases of partial classes, we define the DS dimension more generally for multiclass \textit{partial} concept classes below.
\begin{definition}[DS dimension \cite{daniely2014optimal}]
    \label{def:ds-dim}
    Let $\mcH \subseteq \{ \mcY \cup \{\star\}\}^\mcX$ be a partial concept class and let $S = \{x_1,\dots,x_d\} \in \mcX^d$ be an unlabeled sequence.  For $i \in [d]$, we say that $f,g \in \mcH|_S$ are $i$-neighbours if $f(x_i) \neq g(x_i)$ and $f(x_j) = g(x_j), \; \forall j \neq i$. We say that $\mcH$ DS-shatters $S$ if there exists $\mcF \subseteq \mcH, |\mcF|<\infty$ satisfying
    \begin{enumerate}
    \item $\forall f \in \mcF|_S ,\; \forall i \in [d],$ $f(x_i) \neq \star$.
    \item $\forall f \in \mcF|_S ,\; \forall i \in [d],$ $f$ has at least one $i$-neighbor $g$ in $\mcF|_S$.
    \end{enumerate}
    The DS dimension of $\mcH$, denoted as $d_{\DS}=d_{\DS}(\mcH)$, is the largest integer $d$ such that $\mcH$ DS-shatters\footnote{When $\mcY = \{0,1\}$, DS-shattering is equivalent to the standard notion of VC-shattering i.e., realizability of all binary patterns.} some sequence $S$ of size $d$.
\end{definition}

\subsection{Sample Compression Schemes}
\label{subsec:sample-compression}

The way in which \cite{brukhim2022characterization} construct a learning algorithm for multiclass concept classes having finite DS dimension is through a sample compression scheme. This is sufficient because a successful sample compression scheme implies the existence of a learning algorithm \cite{david2016statistical}. Here, we formally define sample compression schemes.
\begin{definition}[Sample Compression, Definition 29 in \cite{alon2022theory}]
    \label{def:sample-compression}
    A compression scheme $(\kappa, \rho)$ consists of a compression function $\kappa: (\mcX \times \mcY)^* \to (\mcX \times \mcY)^* \times \{0,1\}^*$ and a reconstruction function $\rho: (\mcX \times \mcY)^{*} \times \{0,1\}^{*} \to \mcY^\mcX$. $\kappa$ and $\rho$ must satisfy the following property: for any sequence $S \in (\mcX \times \mcY)^*$, $\kappa(S)=(S', B)$ such that the elements in $S'$ necessarily also exist in $S$.
    The size $k(m)$ of the compression scheme for a given sample size $m$ is 
    \begin{equation}
        k(m)=\max_{S \in (\mcX \times \mcY)^m, (S', B)=\kappa(S)}\max(|S'|, |B|).
    \end{equation}
    The (unqualified) size $k$ of the compression scheme is the maximum size $k(m)$ over all $m$, or infinite if the size can be unbounded.
    
    A compression scheme $(\kappa, \rho)$ is a sample compression scheme for a partial concept class $\mcH \in \{ \mcY \cup \{\star\}\}^\mcX$ if for all finite labeled sequences $S=\{(x_1,y_1),\dots,(x_m,y_m)\}$ realizable by $\mcH$, $\rho(\kappa(S))$ is correct on all of $S$ i.e., $\forall i \in [m],~ \rho(\kappa(S))(x_i)=y_i$. 
\end{definition}
\begin{remark}
    \label{rem:pcc-compression}
    Observe that we only care about compressing sequences \textit{realizable} by the class (no points in the sequence should be labeled with a $\star$), and that the reconstructor $\rho$ always outputs a \textit{total} concept.
\end{remark}

\subsection{Sample Compression Scheme of \texorpdfstring{\cite{moran2016sample}}{msc}}
The $2^{O(d_{\VC})}$-sized compression scheme of \cite{moran2016sample} requires crucially using the \textit{uniform convergence principle} \cite{vapnik2015uniform} and also a bound on the \textit{dual VC dimension} of binary classes having finite VC dimension. These ingredients are combined with a clever application of von Neumann's minimax theorem \cite{v1928theorie} to yield their sample compression scheme. While they are able to use their compression scheme for binary classes in a blackbox manner to derive compression schemes for certain multiclass hypothesis classes having finite \textit{graph dimension}, the graph dimension does not characterize multiclass learnability (more on this in \Cref{subsec:graph-dim-ub}). Instead, we discuss here why their techniques don't translate directly to the multiclass setting in light of the \textit{DS dimension} being the more relevant dimension of interest, as shown by \cite{brukhim2022characterization}.

Firstly, the principle of uniform convergence ceases to hold in the multiclass setting, and the sample complexity of different ERM (Empirical Risk Minimizer) learners can differ by an arbitrarily large factor when the the number of labels is infinite \cite{daniely2015multiclass}. Moreover, the compression scheme of \cite{moran2016sample} crucially makes use of \textit{proper} learners for binary classes i.e., learning algorithms whose output hypotheses always belong to the class. On the other hand, proper learners provably cannot learn multiclass hypothesis classes \cite{daniely2014optimal}! Additionally, for binary classes having VC dimension $d_{\VC}$, the dual VC dimension of the class is bounded above by $2^{d_{\VC}+1}$ \cite{assouad1983densite}. However, in the multiclass setting, the corresponding dual DS dimension may not be bounded above by any finite function of the DS dimension (see \Cref{table:dual-ds} for an illustration)! In particular, we can have every concept in the class use its own set of labels, disjoint from any other concept's labels. If we do this, then it is easy to see that the class cannot DS-shatter any pair of points. However, the dual class can readily DS-shatter arbitrarily large sets.

\begin{table}[H]
    \centering
    \begin{tabular}{ |c|c|c|c|c|c|c|c| } 
     \hline
     \backslashbox{$\barH$}{$\mcX$}& $x_1$ & $x_2$ & $x_3$ & $x_4$ & $x_5$ & $x_6$ & \dots \\
     \hline
     $\barh_1$ & 1 & 1 & 1 & 1 & 1 & 2 &\dots $\in \{1,2\} $\\
     $\barh_2$ & 3 & 3 & 3 & 3 & 4 & 3 &\dots $\in \{3,4\} $\\
     $\barh_3$ & 5 & 5 & 5 & 6 & 5 & 5 &\dots $\in \{5,6\} $\\
     $\barh_4$ & 7 & 7 & 8 & 7 & 7 & 7 &\dots $\in \{7,8\} $\\
     $\barh_5$ & 9 & 10 & 9 & 9 & 9 & 9 & \dots $\in \{9,10\} $\\
     \hline
    \end{tabular}
    \caption{Every hypothesis in $\barH$ uses its own unique set of two labels. Thus, no pair $(x_j, x_k)$ can be DS-shattered by $\barH$. However, notice how the dual class (wherein $\barh_j$'s (rows) become the input domain, and $x_j$'s (columns) form the hypothesis class) can easily realize $i$-neighbors (as in \Cref{def:ds-dim}). In particular, we can see that the dual hypotheses corresponding to $x_2,x_3,x_4,x_5,x_6$ each disagree with the dual hypothesis corresponding to $x_1$ at exactly one of $\barh_1,\barh_2,\barh_3,\barh_4, \barh_5$. We can easily add the remaining dual hypotheses (i.e., columns) to complete the DS-shattering of $\{\barh_1,\barh_2,\barh_3,\barh_4,\barh_5\}$. Similarly, we can even add new rows, each with its own set of labels, and realize arbitrarily large dual DS dimension.}
    \label{table:dual-ds}
\end{table}

Due to these reasons, it seems apparent that the techniques from \cite{moran2016sample} don't port over to the multiclass setting, at least in a straightforward manner. In fact, our main result (\Cref{thm:compression-lb}), which we will now proceed towards proving, shows that this pursuit of constructing a bounded-DS dimension sample compression scheme in the multiclass setting is indeed fruitless. 

\section{Lower Bound via Disambiguation}
\label{sec:lb}

The main ingredient we use to establish \Cref{thm:compression-lb} is the following result from the work of \cite{alon2022theory}, which refutes the sample compression conjecture for binary \textit{partial} concept classes. While \cite{alon2022theory} stated their result only for $\mcY = \{0,1\}$, we can freely think of the label set in their construction to be $\mcY = \{0,1,2,\dots\}$ instead, where we don't use the extra labels available at any point. Walking through their proof pointwise then already gives a lower bound for sample compression in (multiclass) partial concept classes. Moreover, as mentioned above, if we only ever use $\{0,1\}$ labels, DS-shattering is equivalent to VC-shattering. With these considerations, we can state the result from \cite{alon2022theory} in the following form (for completeness, we provide a proof in \Cref{sec:pcc-compression-lb}):

\begin{lemma}[Theorem 7 in \cite{alon2022theory}]
    \label{lem:pcc-lower-bound}
    There exists a partial concept class $\mcH \subseteq \{ \mcY \cup \{\star\}\}^\mcX$ where $\mcY = \{0,1,2,\dots\}$ that has the following properties:
    \begin{enumerate}
        \item $\forall h \in \mcH,~ \forall x \in \mcX,~ h(x) \in \{ 0,1,\star \}$,
        \item $\mcH = \bigcup_{n =1}^\infty H_n$ where each $\mcH_n \subseteq \{ \mcY \cup \{\star\}\}^\mcX$,
        \item $|\mcH_n|, |\supp(\mcH_n)| < \infty$ for every $n$,
        \item $\supp(\mcH_n) \cap \supp(\mcH_m) = \emptyset$ for every $n \neq m$,
        \item $d_{\DS}(\mcH_n) = 1$ for every $n$,
        \item $d_{\DS}(\mcH)=1$,
    \end{enumerate}
    and additionally satisfies: any sample compression scheme for $\mcH$ must have size $\Omega((\log(m))^{1-o(1)})$, where $m$ is the size of an input sequence realizable by $\mcH$, and the $o(1)$ term goes to 0 as $m \to \infty$. In particular, there does not exist a sample compression scheme for $\mcH$ having a size that is a finite function solely of the DS dimension of $\mcH$.
\end{lemma}

Now, we make a simple observation: as a consequence of the definition of disambiguation (\Cref{def:disambiguation}), sequences realizable by a disambiguating total class are necessarily a superset of the sequences realizable by the corresponding partial class. This leads to the following proposition, whose proof is immediate from the definitions of sample compression and disambiguation.
\begin{proposition}[Compression monotonic in disambiguation]
    \label{prop:compression-size-is-monotonic}
    Let $\mcH \subseteq \{ \mcY \cup \{\star\}\}^\mcX$ be a partial concept class and $\barH \subseteq \mcY^\mcX$ be a total concept class that disambiguates $\mcH$. Then, if there exists a sample compression scheme $(\kappa, \rho)$ for $\overline{\mcH}$ of size $k$, then $(\kappa, \rho)$ is also a sample compression scheme for $\mcH$.
\end{proposition}

There is now a natural strategy to prove a sample compression lower bound for total concept classes: find a partial concept class of small DS dimension that is hard to compress to a small size. Then, construct a \textit{disambiguation} of this class to a total concept class. Given the proposition above, the disambiguating class would be at least as hard to compress as the partial class. However, we would want the DS dimension of the disambiguating class to also be small, in order for the lower bound to be meaningful.

The first ingredient in the above strategy is already available to us --- choose the partial class $\mcH$ given by \Cref{lem:pcc-lower-bound} that has DS dimension 1. The crucial task that remains is constructing a disambiguation $\barH$ of $\mcH$ that also has small DS dimension. But given the power of conjuring new labels at will in the multiclass setting, this is not all too hard --- we can construct a disambiguating $\barH$ that also has DS dimension 1 in a straightforward manner. We disambiguate each partial concept $h \in \mcH$ with a total concept $\barh$ that assigns the $\star$'s in $h$ a unique label which is never again used by any other disambiguating concept. This preserves the DS dimension of the disambiguating class.

\begin{lemma}[Disambiguation with no DS blow-up]
    \label{lem:disambiguation}
    There exists a total concept class $\barH \subseteq \mcY^\mcX$ where $\mcY = \{0,1,2,\dots\}$ such that $\barH$ disambiguates $\mcH$ from \Cref{lem:pcc-lower-bound} and satisfies $d_{\DS}(\barH)=1$.
\end{lemma}
\begin{proof}
    Recall that $\mcH = \bigcup_{n =1}^\infty \mcH_n$, where each $|\mcH_n| < \infty$. This means that $\mcH$ is countably large. Let $h_1,h_2,h_3,\dots$ be an enumeration of all the concepts in $\mcH$. Then, for each $i\in\{1,2,\dots\}$, define the total concept $\barh_i:\mcX \to \mcY$ as follows
    \begin{align*}
        \barh_i(x) = \begin{cases}
            h_i(x) & \text{if } h_i(x) \neq \star, \\
            i+1 & \text{otherwise.}
        \end{cases}
    \end{align*}
    Consider the total concept class $\barH = \bigcup_{i=1}^\infty \barh_i$. By construction, $\barH$ disambiguates $\mcH$. Further, any sequence DS-shattered by the partial class is certainly also DS-shattered by $\barH$, and hence $d_{\DS}(\barH)\ge 1$. Now, let $S= \{x_1,\dots,x_d\}$ be any sequence that is DS-shattered by $\barH$. Then, according to \Cref{def:ds-dim} above, let $\barF = \{\barf_1,\dots\barf_m\} \subseteq \barH$ be the finite subset of $\barH$ that realizes this shattering. Namely, if we think of all the distinct patterns $\barf_1|_S,\barf_2|_S,\dots,\barf_m|_S$ that $\barF$ realizes on $S$, then every pattern $\barf_i|_S$ has a neighbor $\barf_j|_S$ in every direction $l \in [d]$ ($\barf_i|_S$ and $\barf_j|_S$ are the same everywhere but at index $l$). There can be two cases: either every string $\barf_i|_S$ is such that $\barf_i|_S \in \{0,1\}^d$. But this would mean that the partial concepts $\{f_1,\dots,f_d\}$ themselves realize the DS-shattering of $S$, implying that $d\le 1$. In the other case, there is some $\barf_i$ which satisfies $\barf_i(x_l) \in \{2,3,\dots\}$ for some $l \in [d]$. But now observe that no other function in $\barH$ other than $\barf_i$ attains the label $\barf_i(x_l)$ on $x_l$ --- this is merely an artefact of our construction of $\barH$. In particular, this means that $\barH$ cannot realize a neighbor for $\barf_i|_S$ in any direction other than $l$, meaning also that $d \le 1$. This completes the proof that $d_{\DS}(\barH) \le 1$.
\end{proof}

\begin{remark}
    \label{rem:proof-of-lb}
    \Cref{lem:disambiguation}, \Cref{prop:compression-size-is-monotonic} and \Cref{lem:pcc-lower-bound} together imply \Cref{thm:compression-lb}.
\end{remark}

\section{Discussion}
\label{sec:discussion}

We showed that unlike the binary setting, compression and learnability are not equivalent in the multiclass learning setting. Namely, if we are allowed infinite labels, it is possible that a hypothesis class is learnable, but the size of a compressed sample must necessarily scale with the size of the original sample, and cannot be independent of it. Our result illustrates a separation between the paradigms of compression and learnability in the binary and multiclass settings. In the following, we discuss the relevance of our result in the context of a past result on multiclass compression by \cite{moran2016sample}, and also discuss why the disambiguation technique from above does not work in order to prove lower bounds for sample compression when the disambiguating class is only allowed to use finitely many labels.

\subsection{Upper Bound from \texorpdfstring{\cite{moran2016sample}}{ub} in Terms of Graph Dimension}
\label{subsec:graph-dim-ub}

As mentioned above, the seminal work of \cite{moran2016sample} answered the question ``does learnability imply compression?" in the affirmative for binary hypothesis classes. Namely, for any binary hypothesis class of VC dimension $d_{\VC}$, \cite{moran2016sample} construct a sample compression scheme of size $2^{O(d_{\VC})}$. In fact, there is a nice reduction (outlined in \Cref{sec:graph-dim-compression} for completeness) from the multiclass setting to the binary setting that allows them to use their compression scheme as is, and obtain a sample compression scheme of size $2^{O(d_{G})}$ for any (multiclass) hypothesis class, where $d_G$ is the \textit{graph dimension} of the class. The graph dimension $d_G$ is defined as follows:

\begin{definition}[Graph dimension \cite{natarajan1989learning}]
    \label{def:graph-dim}
    Let $\barH \subseteq \mcY^\mcX$ be a hypothesis class and let $S \in \mcX^d$ be a sequence. We say that $\barH$ G-shatters $S$ if there exists an $\barf \in \barH$ (which realizes the G-shattering), such that for every subsequence $T \subseteq S$, there exists an $\barh \in \barH$ such that
    \begin{align*}
        \forall x \in T,~ \barh(x)=\barf(x), \; \text{and } \forall x \in S \setminus T,~ \barh(x) \neq \barf(x).
    \end{align*}
    The size of the largest sequence that $\barH$ G-shatters\footnote{Note that just like DS-shattering, G-shattering is also equivalent to VC-shattering when $\mcY = \{0,1\}$.} is called the graph dimension of $\barH$, denoted as $d_G(\barH)$.
\end{definition}
 However, even if it is the case that a hypothesis class always permits a sample compression scheme of size $2^{O(d_{G})}$, the graph dimension $d_G$ need not necessarily be finite for a learnable hypothesis class when the label space is allowed to be infinitely large. In particular, \cite{daniely2014optimal} constructed a learnable hypothesis class $\barH$ that has $d_G(\barH)=\infty$. This is precisely why the compression scheme of \cite{moran2016sample} in terms of the graph dimension does not allow them to immediately conclude that ``learnability implies compression" in the multiclass case. Furthermore, our lower bound does not contradict their upper bound. This is because the disambiguation we construct in \Cref{lem:disambiguation} above only preserves the DS dimension\footnote{and also, the Natarajan dimension \cite{natarajan1989learning}, which is also required to be finite for PAC learnability.}. On the other hand, the graph dimension of the disambiguating class can (and must) increase arbitrarily, so that the $2^{O(d_{G})}$ bound is still a valid (but not meaningful) upper bound on the size of the compression scheme. Here is a simple example to see why this might be the case: fix a sequence $S=\{x_1,x_2,x_3\}$, and say that some $h$ in the partial class $\mcH$ realizes the pattern $(0,0,0)$ on this sequence. Say also that the set of distinct patterns that is realized on $S$ by the rest of the partial concepts in $\mcH$ is $(\star,0,0), (0,\star,0), (0,0,\star), (\star,\star,0), (\star,0,\star), (0,\star,\star), (\star,\star,\star)$. Observe that this sequence is not remotely DS/VC-shattered by the partial class. Now, let us think of the patterns that the disambiguating class realizes on this sequence. Since each total concept in the disambiguating class labels the $\star$'s in the partial concept it represents with a distinct number, the patterns on $S$ realized by $\barH$ would be something like:
\begin{alignat*}{6}
    &(0,0,0) &&\longrightarrow\; && (0,0,0) &&\qquad\qquad\qquad(\star,\star,0) &&\longrightarrow\; && (20,20,0) \\
    &(\star,0,0) &&\longrightarrow\; && (3,0,0) &&\qquad\qquad\qquad(\star,0, \star) &&\longrightarrow\; && (39,0,39) \\
    &(0,\star,0) &&\longrightarrow\; && (0,7,0) &&\qquad\qquad\qquad (0,\star,\star) &&\longrightarrow\; && (0,53,53) \\
    &(0,0,\star) &&\longrightarrow\; && (0,0,11) &&\qquad\qquad\qquad(\star,\star,\star) &&\longrightarrow\; && (100, 100, 100).
\end{alignat*}
We can readily see that $S$ is G-shattered by $\barH$ ($\barh$ realizes the G-shattering). This phenomenon must indeed be occurring at a larger scale --- arbitrarily large sequences must be getting G-shattered by $\barH$ in the manner illustrated above, so as to ensure that the $2^{O(d_{G})}$ upper bound does not contradict our lower bound from \Cref{thm:compression-lb}.

\subsection{Upper/Lower Bounds in Terms of Natarajan Dimension for Finite Labels}
\label{subsec:natarajan-dim-ub-lb}
In a sense, the DS dimension really only comes into picture when dealing with hypothesis classes that have an infinite label space. When the label space of the hypothesis class is finite, i.e., $|\mcY|=c<\infty$, learnability of the class is completely characterized (e.g., \cite[Theorem 4]{daniely2015multiclass}) by another folklore quantity called the \textit{Natarajan dimension} $d_N$. The Natarajan dimension of a class must unconditionally be finite for learnability (irrespective of finitely many/infinite labels); however, its finiteness is sufficient for learnability only if the label space is finite. The Natarajan dimension is defined as follows:

\begin{definition}[Natarajan dimension \cite{natarajan1989learning}]
    \label{def:natarajan-dim}
    Let $\barH \subseteq \mcY^\mcX$ be a hypothesis class and let $S \in \mcX^d$ be a sequence. We say that $\barH$ N-shatters $S$ if there exist $\barf_1, \barf_2 \in \barH$ (which realize the $N$-shattering), such that $\forall x \in S,~ f_1(x) \neq f_2(x)$, and further, for every subsequence $T \subseteq S$, there exists an $\barh \in \barH$ such that
    \begin{align*}
        \forall x \in T,~ \barh(x)=\barf_1(x), \; \text{and } \forall x \in S \setminus T,~ \barh(x) = \barf_2(x).
    \end{align*}
    The size of the largest sequence that $\barH$ N-shatters\footnote{Again, N-shattering is equivalent to VC-shattering when $\mcY = \{0,1\}$.} is called the Natarajan dimension of $\barH$, denoted as $d_N(\barH)$.
\end{definition}

Observe that a sequence that is N-shattered by $\barH$ is also G-shattered by it, implying $d_N \le d_G$. Additionally, due to a result by \cite{bendavid1995characterizations}, the graph dimension can also be upper-bounded in terms of the Natarajan dimension and the number of classes $c$, as
\begin{align}
    d_N \le d_G \le O(d_N\cdot\log(c)).
\end{align}
This relation, combined with the sample compression scheme due to \cite{moran2016sample} above, immediately implies a sample compression scheme of size $c^{O(d_N)}$ for any hypothesis class on $c$ classes having Natarajan dimension $d_N$. Since finiteness of the Natarajan dimension is a necessary condition for leanability, we conclude that compression and learnability are in fact equivalent in the multiclass setting \textit{when the number of labels is finite}.

As for lower bounds, a straightforward lower bound can be obtained by a counting argument, similar to \cite[Theorem 14]{floyd1995sample}. Concretely, let $\barH \subseteq \mcY^\mcX$ such that $|\mcY|=c$ and $|\mcX|=m$. For any sample compression scheme of size $k$, the number of distinct ``compression sets'' possible are at most $\sum_{i=0}^k\binom{m}{i}\cdot c^i \cdot 2^{\Theta(k)}$ (choose $i$ distinct elements from $\mcX$, label it in one of at most $c^i$ possible ways, and append an additional bit string of size at most $\Theta(k)$ to it), which is at most $\left(\frac{cme}{k}\right)^{\Theta(k)}$. If we now think of compressing the entire domain of each hypothesis in the class, each of these compression sets should point to a distinct hypothesis in the class, and hence there should at least be one compression set for every hypothesis in the class. Consequently, if the size of $\barH$ were to be large, while keeping its Natarajan dimension bounded, we would get a lower bound on $k$. As \cite[Theorem 4]{haussler1995generalization} show, for any $c, d_N \le m$, one can construct $\barH \subseteq \mcY^\mcX$ having Natarajan dimension $d_N, |\mcY|=c$ and $|\mcX|=m$ such that
\begin{align*}
    |\barH| = \sum_{i=0}^{d_N}\binom{m}{i}(c-1)^i.
\end{align*}
If we set $m=d_N$, we get $|\barH|=c^{d_N}$. Recalling the upper bound on the number of compression sets from before, we can conclude that if $k < C\cdot d_N$ for some absolute constant $C$, the number of possible compression sets will be smaller than $|\barH|$. In summary, the size $k$ of any valid sample compression must satisfy
\begin{align*}
\Omega(d_N) \le k \le c^{O(d_N)}.
\end{align*}
Note that setting $c=2$ (which also makes $d_N=d_{\VC}$) recovers the longstanding unsettled exponential gap between upper and lower bounds in the size of compression schemes in the binary case. For larger but constant $c$, we are morally faced with the same unsettled exponential gap, where the Natarajan dimension replaces the VC dimension. Perhaps unsettling is the regime where we think of $d_N$ as constant. In this case, note that the lower bound is $\Omega(1)$ and does not even depend on the number of classes $c$, whereas the upper bound is $\poly(c)$. It seems plausible that the lower bound on the compression size should grow with the number of classes.
\begin{openproblem}
    \label{openprob:multiclass-lb-finite-classes}
    Let $k$ be the (unqualified) size of any valid sample compression scheme for a hypothesis class $\barH \subseteq \mcY^\mcX$ having Natarajan dimension $d_N=\Theta(1)$, where $|\mcY|=c<\infty$. Is $k = \omega(1)$ with respect to $c$? Is $k=\Omega(\polylog(c))$? Is $k=\Omega(c^\delta)$ for some $\delta > 0$?
\end{openproblem}

\subsection{Disambiguation Using Only Finitely Many Labels}
\label{subsec:disambiguation-using-finite-labels}
In the disambiguation that we constructed above, we crucially used the power of \textit{infinite} labels available to us. In fact, using infinite labels is necessary for this proof technique. If instead, we only considered disambiguations that label $\star$'s with one of $c<\infty$ labels, we cannot hope to preserve learnability of the disambiguating class. For $c=2$, this is immediate from Theorem 1 in \cite{alon2022theory}, which says that any binary (total) concept class disambiguating the partial class from \Cref{lem:disambiguation} \textit{must} have infinite VC dimension. However, even for $c>2$ but finite, this approach will not work. This is crucially because any total class $\barH$ disambiguating the partial class $\mcH$ from \Cref{lem:pcc-lower-bound} also disambiguates each of the $\mcH_n$s individually for increasing $n$. We can then instantiate \Cref{lem:disambiguation-pcc-large} in \Cref{subsec:pcc-hard-to-disambiguate} and the contrapositive of the \textit{multiclass} version of the Sauer-Shelah-Perles lemma \cite{haussler1995generalization} to conclude that $d_N(\barH)=\infty$. Thus, this approach with finite-label disambiguators will only let us derive a lower bound on the compression size for what has become an \textit{unlearnable} class, a not-so-interesting result. In contrast, and perhaps intriguingly, disambiguating with \textit{infinite} labels allows us to retain the learnability of the disambiguating class, while also inheriting the lower bound on the compression size from the underlying partial class.

\section*{Acknowledgements}
I am deeply grateful to Moses Charikar for his discussions and valuable feedback on the manuscript. This work is supported by Moses' Simons Investigator Award, and NSF awards 1704417 and 1813049.
\bibliographystyle{alpha}
\bibliography{references}

\appendix
\section{Proof of \texorpdfstring{\Cref{lem:pcc-lower-bound}}{ll}}
\label{sec:pcc-compression-lb}
In this section, we essentially rewrite the proof by \cite{alon2022theory}, which is a beautiful reduction from a recent breakthrough result by \cite{balodis2022unambiguous}, which in turn builds upon the works of \cite{ben2017low, goos2015lower} to show a nearly tight bound for the Alon-Saks-Seymour problem \cite{kahn1991recent, bousquet2014clique}.\\ \cite{alon2022theory} translate the result by \cite{balodis2022unambiguous} into the construction of a partial concept class that is hard to disambiguate with a small total class, and also consequently hard to compress.

\subsection{A Partial Class That is Hard to Disambiguate}
\label{subsec:pcc-hard-to-disambiguate}
 We recall some concepts from graph theory. Given a graph $G=(V,E)$, the chromatic number $\chi(G)$ of $G$ is the minimum number of colors required, such that each vertex can be assigned a color in a way that no two vertices connected by an edge have the same color assigned to them. The biclique partition number $\BP(G)$ is the minimum number of complete bipartite graphs required to successfully partition the edge set $E$ of $G$. Each complete bipartite graph in the decomposition consists of all the vertices in $V$ (some possibly isolated, but the rest forming a complete bipartite graph) and a subset of the edges in $E$. \cite{balodis2022unambiguous} proved the following result relating these two quantities, in response to a problem originally posed by Alon, Saks and Seymour \cite{kahn1991recent}:

\begin{theorem}[Corollary 3 in \cite{balodis2022unambiguous}]
    \label{thm:chromatic-number-biclique-partition-number}
    For every $n$, there exists a finite simple graph $G=(V, E)$ with $\BP(G)=n$ such that
    \begin{align*}
        \chi(G) \ge n^{(\log(n))^{1-o(1)}},
    \end{align*}
    where the $o(1)$ term goes to $0$ as $n \to \infty$.
\end{theorem}
Now, we describe the clever reduction by \cite{alon2022theory}, who leverage the lower bound result above to construct a partial concept class that is hard to disambiguate with a small total concept class. Given $n$, let $G=(V,E)$ be the graph promised by \Cref{thm:chromatic-number-biclique-partition-number}, and let $B_i=(L_i, R_i, E_i)$ be $n$ complete bipartite graphs (identified with numbers in $[n]$) that witness the partitioning of the edge set of $G$ such that $\BP(G)=n$, and the edge sets $E_i$ are pairwise disjoint. Let $\mcY=\{0,1,2,\dots\}$, and define the multiclass partial concept class $\mcH_n \subseteq \{\mcY \cup \{\star\}\}^{[n]}$ as follows: for each vertex $v \in V$, $\mcH_n$ contains a partial concept $h_v$ such that for each $i \in [n]$,
\begin{equation}
    \label{eqn:pcc-def}
    h_v(i) = \begin{cases}
        0 & \text{if } v \in L_i, \\
        1 & \text{if } v \in R_i, \\
        \star & \text{otherwise}.
    \end{cases}
\end{equation}
Since $\{0,1\}$ are the only non-$\star$ labels in $\mcH_n$, DS-shattering reduces to VC-shattering. We have the following two lemmas:
\begin{lemma}[Lemma 31 in \cite{alon2022theory}]
    \label{lem:vc-pcc-1}
    $\DS(\mcH_n)=1$.
\end{lemma}
\begin{proof}
    Since there exists at least one edge in $G$, the hypotheses corresponding to the endpoints of this edge shatter a set of size 1, and hence $\DS(\mcH_n) \ge 1$. We will show that for any $i \neq j$, $\mcH_n$ cannot simultaneously realize the patterns $(0,0)$ and $(1,1)$ on $(i,j)$, implying that $\DS(\mcH_n) < 2$. Towards a contradiction, assume that some $h_u$ satisfies $h_u(i)=0$ and $h_u(j)=0$. From $\eqref{eqn:pcc-def}$ above, this means that $u \in L_i$ and $u \in L_j$. Now, assume also that some $h_v$ satisfies $h_v(i)=1$ and $h_v(j)=1$. This means that $v \in R_i$ and $v \in R_j$. But since the bipartite components $B_i$ are complete, this means that the edge $(u, v)$ exists in both $B_i$ and $B_j$, contradicting the disjointedness of the edge sets $E_i$ and $E_j$.
\end{proof}

\begin{lemma}[Lemma 32 in \cite{alon2022theory}]
    \label{lem:disambiguation-pcc-large}
    Let $\barH_n \in \mcY^{[n]}$ be a total concept class that disambiguates $\mcH_n$. Then, $\barH_n$ defines a coloring of $G$ using $|\barH_n|$ colors. Therefore, from \Cref{thm:chromatic-number-biclique-partition-number} above,
    \begin{align*}
        |\barH_n| \ge n^{(\log(n))^{1-o(1)}}.
    \end{align*}
\end{lemma}
\begin{proof}
    Let $\barh_v \in \barH_n$ be the total concept that disambiguates $h_v \in \mcH_n$ i.e., every sequence realizable by $h_v$ is also realizable by $\barh_v$. Then, we identify the concept $\barh_v$ with a unique color $Id(\barh_v)$, and assign the vertex $v$ this color. This defines a candidate coloring of the vertices of $G$. It remains to argue that no two endpoints of any edge in $G$ are assigned the same color. Indeed, let $(u, v)$ be en edge in $G$. Then, this edge necessarily exists in exactly one of the bipartite components $B_i=(L_i, R_i, E_i)$, meaning that either $h_u(i)=0, h_v(i)=1$ or $h_u(i)=1, h_v(i)=0$. Whatever be the case, $h_u$ and $h_v$ necessarily disagree on $i$, and therefore so do their disambiguators $\barh_u$ and $\barh_v$, implying $Id(\barh_u)\neq Id(\barh_v)$.
\end{proof}

Now, for each $n$, we can instantiate $\mcH_n$ as defined above, each having its own separate domain, and extend the domain of every $\mcH_n$ to the union of the domains as follows: every $h \in \mcH_n$ labels the domain of any other $\mcH_m$ entirely with $\star$. For the domain-extended $\mcH_n$'s thus defined, by construction, we have that $\supp(\mcH_n) \cap \supp(\mcH_m) = \emptyset$ for all $n \neq m$. Furthermore, since each $\mcH_n$ has a support of size $n$ and is based on a finite simple graph, we already have $|\mcH_n|, |\supp(\mcH_n)|<\infty$. Let $\mcH = \bigcup_{n=1}^\infty \mcH_n$. Since any shattered sequence would have to entirely lie in the support of a single $\mcH_n$, by \Cref{lem:vc-pcc-1} above, we also have that $\DS(\mcH)=1$. This justifies all the points describing $\mcH$ in \Cref{thm:compression-lb}.

\subsection{Compression Implies Disambiguation}
\label{subsec:compression-implies-disambiguation}
The following lemma shows that sample compression schemes imply disambiguations of bounded size for partial concept classes.
\begin{lemma}[Proposition 14 in \cite{alon2022theory}]
    \label{lem:compression-implies-disambiguation}
    For $\mcY = \{0,1,2,\dots\}$, let $\mcH \in \{\mcY \cup \{\star\}\}^\mcX$ be such that 
    \begin{enumerate}
        \item $\forall h \in \mcH,~ \forall x \in \mcX,~ h(x) \in \{0,1,\star\}$,
        \item $|\supp(\mcH)| \le n$.
    \end{enumerate}
    Then, if $(\kappa, \rho)$ is a sample compression scheme for $\mcH$ of (unqualified) size $k$, then there exists a disambiguation of $\mcH$ of size at most $n^{\Theta(k)}$.
\end{lemma}
\begin{proof}
    By definition of a compression scheme, $\kappa$ must be able to compress the support of every $h \in \mcH$ to a short labeled sequence of size at most $k$, such that the output of $\rho$ on this short sequence (together with some appropriate bit string of size at most $\Theta(k)$) correctly labels the entire support of $h$. Thus, if we iterate over all possible realizable sequences and bit strings of size at most $\Theta(k)$, the reconstruction by $\rho$ on all of these necessarily disambiguates every single partial concept in $\mcH$. 
    Since $|\supp(\mcH)| \le n$, the total number of configurations that we need to apply $\rho$ to is at most 
    $\sum_{i=0}^k\binom{n}{i}\cdot2^i\cdot2^{\Theta(k)}$ (choose $i$ distinct elements from $\supp(\mcH)$, label it in one of $2^i$ possible ways, and append a bit string of size at most $\Theta(k)$ to it), 
    which is at most $n^{\Theta(k)}$ as required.

\end{proof}

\subsection{Putting Things Together}
\label{subsec:put-together}
Say there exists a compression scheme $(\kappa, \rho)$ for $\mcH$ defined in \Cref{subsec:pcc-hard-to-disambiguate} above that compresses labeled sequences of size $m$ to size $k(m)$. Then, observe that $(\kappa, \rho)$ defines a compression scheme of (unqualified) size $k=k(m)$ for $\mcH_m$ (for any sequence realizable by $\mcH_m$, which must be of size at most $m$, elongate the sequence (if required) to have size exactly $m$ with duplicate elements, and $(\kappa, \rho)$ now correctly compresses-reconstructs it). From \Cref{lem:compression-implies-disambiguation} above, this implies a disambiguation of $\mcH_m$ of size at most $m^{\Theta(k(m))}$. But then, \Cref{lem:disambiguation-pcc-large} necessitates that
\begin{align*}
    m^{\Theta(k(m))} \ge m^{(\log(m))^{1-o(1)}}
\end{align*}
which gives us that $k(m) = \Omega((\log(m))^{1-o(1)})$ as required.

\section{Sample Compression Scheme in Terms of Graph Dimension}
\label{sec:graph-dim-compression}
We elaborate on the reduction from sample compression schemes for binary hypothesis classes to those for multiclass hypothesis classes from Section 4.1 from \cite{moran2016sample}.
Given a hypothesis class $\barH: \mcX \to \mcY$ having graph dimension $d_G$, construct the \textit{binary} hypothesis class $\barH': (\mcX \times \mcY) \to \{0,1\}$, defined as follows: $\barH' = \{\barh' : \barh \in \barH \}$ where $\barh'$ is defined as follows:
\begin{align*}
    \barh'(x,y) = \begin{cases}
        1 & \text{if } \barh(x) = y, \\
        0 & \text{otherwise.}
    \end{cases}
\end{align*}
We can then see that $d_{\VC}(\barH') = d_G(\barH)$. Now, given a sample $S= \{(x_1,y_1),\dots,(x_m,y_m)\}$ realizable by $\barH$, the compressor constructs the sequence $S' = \{((x_1,y_1),1),\dots,((x_m,y_m),1) \}$ that is realizable by $\barH'$. Theorem 1.4 in \cite{moran2016sample} then implies that there exists a sample compression scheme of size $2^{O(d_G)}$ for $\barH'$. Namely, let $\ERM_{\barH'}(S')$ be any hypothesis in $\barH'$ entirely consistent with $S'$. Then, there exist subsequences $S'_1,\dots, S'_{t}$ of $S'$ (where $t = O(2^{d_G})$) each of size $O(d_G)$ such that the majority vote of $\ERM_{\barH'}(S'_1), \dots, \ERM_{\barH'}(S'_{t})$ is 1 on every $(x_i,y_i)$ pair in $S'$. This equivalently means that the majority vote of $\ERM_{\barH}(S_1),\dots,\ERM_{\barH}(S_t)$ is the correct label $y_i$ for every $x_i$ in $S$. Thus, the compressor compresses $S$ to $S_1,\dots,S_t$ (along with a bit string of size $2^{O(d_G)}$ specifying splits), and the reconstructor invokes ERM on each $S_i$ and takes the majority vote to obtain correct predictions on all of $S$.

\end{document}